\documentclass[submission,copyright,creativecommons,noncommercial]{eptcs}

\providecommand{\event}{TARK 2017} % Name of the event you are submitting to
\usepackage{underscore}            % Only needed if you use pdflatex.

\usepackage[T1]{fontenc}
\usepackage[english]{babel}
\usepackage[utf8]{inputenc}
\usepackage{amsthm}
\usepackage{amsmath}
\usepackage{amsfonts}
\usepackage{amssymb}
\usepackage{stmaryrd}
\usepackage{booktabs}
\usepackage[dvipsnames]{xcolor}
\usepackage{verbatim}
\usepackage{paralist}
\usepackage{tikz}
\usepackage{chngcntr}
\usepackage{hyperref}
\usepackage{cleveref}

\title{Argument-based Belief in Topological Structures}

\author{%
  Chenwei Shi \qquad Sonja Smets \qquad Fernando R. Vel{\'a}zquez-Quesada
  \institute{Institute for Logic, Language and Computation\\ Amsterdam, The Netherlands}
  \email{ \{C.Shi, S.J.L.Smets, F.R.VelazquezQuesada\}@uva.nl}
}

\def\titlerunning{Argument-based Belief in Topological Structures}
\def\authorrunning{C. Shi, S. Smets \& F.R. Vel{\'a}zquez-Quesada}

\counterwithin{figure}{section}

\usetikzlibrary{positioning}
\usetikzlibrary{graphs}
\usetikzlibrary{arrows}

\tikzset{
  argument/.style = {
    draw,
    circle,
    thick,
    inner sep=0pt,
    minimum size=6mm
  }
}

\hypersetup{
  pdfstartview      = {FitH},
  pdfpagelayout     = {OneColumn},
  bookmarksnumbered = {true},
  naturalnames      = {true},
}

\newtheorem{theorem}{Theorem}[section]

\newtheorem{pro}{Proposition}[section]
\newtheorem{lem}{Lemma}[section]

\newtheorem{dfn}{Definition}[section]
\newtheorem{example}{Example}[section]

\newcommand{\lt}{\leftarrowtail}

\DeclareMathOperator{\LFP}{\mathsf{LFP}}
\DeclareMathOperator{\Bel}{\mathit{Bel}}
\DeclareMathOperator{\B}{\mathcal{B}}
\DeclareMathOperator{\NB}{\mathsf{B}}
\DeclareMathOperator{\PB}{\mathbf{B}}

\newcommand{\set}[1]{\{#1\}}

\newcommand{\br}[1]{\llbracket #1\rrbracket}

\newcommand{\leftright}{\ensuremath{\boldsymbol{(\Rightarrow})}}
\newcommand{\rightleft}{\ensuremath{\boldsymbol{(\Leftarrow})}}
\newcommand{\prsubset}{\ensuremath{\boldsymbol{(\subseteq})}}
\newcommand{\prsupset}{\ensuremath{\boldsymbol{(\supseteq})}}

\newenvironment{ap}[1]
{ \smallskip\noindent{\scshape #1} }
{ \smallskip }

\newcommand{\ourparagraph}[1]{\medskip\noindent\textit{#1}}

\newcommand{\shi}[1]{#1}
\newcommand{\fer}[1]{#1}

\begin{document}

\maketitle

\begin{abstract}
  This paper combines two studies: a topological semantics for epistemic notions and abstract argumentation theory. In our combined setting, we use a topological semantics to represent the structure of an agent's collection of evidence, and we use argumentation theory to single out the relevant sets of evidence through which a notion of \emph{beliefs grounded on arguments} is defined. We discuss the formal properties of this newly defined notion, providing also a formal language with a matching modality together with a sound and complete axiom system for it. Despite the fact that our agent can combine her evidence in a `rational' way (captured via the topological structure), argument-based beliefs are not closed under conjunction. This illustrates the difference between an agent's reasoning abilities (i.e. the way she is able to combine her available evidence) and the closure properties of her beliefs. We use this point to argue for why the failure of closure under conjunction of belief should not bear the burden of the failure of rationality.
\end{abstract}

%********************************************************************************************************
\section{Introduction}
%********************************************************************************************************

The concept of belief has been extensively studied in philosophy, logic and computer science. The representations range from purely qualitative structures including the relational models with serial, transitive and Euclidean relations (the \textit{KD45} representation in doxastic logic; \cite{hintikka1962,MeyeretAl}), plausibility models \cite{AS2008,Board2002} and topological models \cite{Baltag2013,Baltag2015report,Baltag2016} to quantitative structures including the ranking-based plausibility representation \cite{Spohn1988}, Bayesian models using subjective probability functions and conditional probabilistic spaces designed to represent conditional belief \cite{Baltag2008}.

Despite the success of the early qualitative and quantitative approaches\fer{, most of them} had one important disadvantage: they represent only what the agent believes, without dealing with the reasons, justifications or evidence on which such beliefs are grounded.\footnote{\fer{An important exception are the so called \emph{truth maintenance systems} \cite{doyle:truth:1979,Kleer86} which, inspired by the typical forms of arguments in systems for inference in natural deduction, keep track of \emph{syntactic} justifications.}} To remedy this shortcoming, the last decade witnessed a number of different frameworks that can portray the evidence or justifications on which different epistemic attitudes are based. For instance, the work in \cite{BRS1,BRS2}, combines dynamic epistemic logic with justification logic in the tradition of \cite{Artemov08} and \cite{ArtemovN05} where justifications are syntactic terms. In \cite{BenthemP11} and \cite{BenthemFP14} a semantic approach is adopted, representing evidence as a set of possible worlds so beliefs are defined by the maximally consistent ways in which evidence can be combined (so called \emph{evidence models}). The work in \cite{Baltag2016} follows the latter direction, adding a topological structure which represents the different ways in which the available pieces of evidence can be combined, and then using topological notions to single out relevant sets of \emph{combined} pieces of evidence.

The present contribution follows the mentioned topological approach \cite{Baltag2016}, where we view arguments as special types of evidence that an agent possesses. By combining the topological approach to evidence logic with the tools from abstract argumentation theory \cite{dung1995}, we can single out the sets of pieces of evidence (or arguments) on which a new notion of \emph{belief grounded on arguments} can be defined. This new concept of {\it grounded belief} provides the bridge between an agent's doxastic attitude and the arguments she has at her disposal.  While argumentation theorists such as \cite{dung1995} informally talk about the relation between beliefs and arguments, the formalization of the details has only recently gained more attention (e.g., \cite{schw2012,grossi2014}). The approach in this paper is related to the line of work which studies abstract argumentation theory by using the tools and techniques of modal logic \cite{grossi2013,caminada:logical:2009}, and to the work by \cite{grossi2014} which introduces a notion of justified belief by combining doxastic logic and abstract argumentation theory. Yet in contrast to the work on relational models for modal logic, we do turn to the topological semantics of \cite{Baltag2016}. 

\fer{Due to its semantic nature, one might expect for this paper's notion of grounded belief to be very much like the notions of belief studied in relational and evidence models, but actually this is not the case. Grounded beliefs turn out to be not closed under conjunction, which makes them} more similar to the probabilistic notions of belief as studied in \cite{lenzen:glauben:1980} and \cite{van:eijck:belief:2014}. In the final part of this paper we come back to this analysis and provide a comparison between our concept of grounded belief and its quantitative counterpart. From a philosophical perspective the lack of closure raises a number of interesting issues. It shows how, in our framework, there is a clear difference between the agent's reasoning abilities (i.e., the way she is able to combine her available evidence, represented by the topological structure) and the closure properties of her beliefs (related to the strategy she uses to select the sets of combined evidence on which her beliefs will be based). This relates to a point which is raised by \cite{foley:beliefs:2009} in the context of the Lottery paradox. \cite{foley:beliefs:2009} describes a dividing line between processes of deductive reasoning (which are closed under conjunction) and those which involve beliefs and are not closed under conjunction. While we agree with Foley that there are essential differences in such reasoning patterns, our paper does show how the two can be made formally precise and can moreover be united in a formal framework that defines grounded beliefs on the basis of deductive reasoning processes for combining evidence.

\ourparagraph{Outline.} Section \ref{sec:preli} recalls the work on evidence and belief that our proposal is based on \cite{BenthemP11,Baltag2016}. Section \ref{sec:GAM} introduces our framework, defining our notion of \emph{belief grounded on arguments} and comparing it with the topological notion of evidence-based belief in \cite{Baltag2016}. Section \ref{sec:ABBL} provides a sound and complete logic characterizing the properties of our notion of belief. Section \ref{sec:failclosure} compares our setting with the probabilistic notion of belief, re-examining the criticism faced by notions of belief which are not closed under conjunction. We conclude in Section \ref{sec:conclusions}, listing several directions for future work.

%********************************************************************************************************
\section{Preliminaries}\label{sec:preli}
%********************************************************************************************************

In this section we recall the basic concepts and definitions of the framework in \cite{BenthemP11} designed to present evidence logic as well as the further topological development presented in \cite{Baltag2016}.

\ourparagraph{Beliefs in Evidence Models} The evidence an agent has for supporting her beliefs can be represented in different ways. The proposal in \cite{BenthemP11} takes a semantic perspective and uses neighborhood models \cite{Scott1970,Montague1970}, a generalisation of relational models, in which a so-called \emph{neighborhood function} assigns a family of subsets of the domain to each possible world.\footnote{For our purposes, we work only with the so-called \emph{uniform} evidence models in \cite{BenthemP11}, in which the neighborhood function is world-invariant.}\;\footnote{\fer{As Definition \ref{def:evidence_model} shows, the concept of evidence represented in evidence models is rather abstract: a piece of evidence is understood as a (non-empty) set of possible worlds. As it will be discussed, these pieces of evidence can be understood as information the agent receives from external sources (observations, communication), and the closure properties the evidence set might have can be understood as the different ways she can `extract' further information from what she gets. For an in-depth study about the different interpretations the term \emph{evidence} and its `cousin' \emph{justification} might have, the reader is referred to \cite{sep-evidence} (for the first) and \cite{sep-justep-intext} (for the second).}}
% also, sep-justep-foundational,sep-justep-coherence

\begin{dfn}[Evidence model \cite{BenthemP11}]\label{def:evidence_model}
  An \emph{evidence model} is a tuple $\mathcal{M} = (X,E_0,V)$ where $X \neq \emptyset$ is a set of possible worlds, $E_0 \subseteq 2^X-\{\emptyset\}$ is a family of non-empty sets called \emph{pieces of evidence}, and $V: \mathsf{Prop}\rightarrow 2^W $ is a valuation function for a given set $\mathsf{Prop}$ of atomic propositions. The family $E_0$ is required to include $X$ (i.e. $X \in E_0$).
\end{dfn}

Note that the above definition mentions two explicit constraints on evidence models: $X$ is always in the neighborhood (the whole domain is always a piece of evidence) but $\emptyset$ is never there (evidence per se is never contradictory). As these are the only possible constrains, it can happen that two pieces of evidence in $E_0$ are in conflict with each other (i.e., there might be $e_1, e_2 \in E_0$ such that $e_1 \cap e_2 = \emptyset$). Moreover, in this setting $E_0$ is not required to satisfy any closure property, e.g. closure under finite/arbitrary unions and/or intersections. Still, the lack of closure properties in \emph{the family} of pieces of evidence does not imply that \emph{the agent} is not able to combine the evidence pieces in a meaningful way. Following the terminology in \cite{BFS2016,Baltag2016}, we introduce the notion of a body of evidence:

\begin{dfn}[Body of evidence]\label{def:body_evidence}
  Let $\mathcal{M} = (X,E_0,V)$ be an evidence model. A \emph{body of evidence $F \subseteq E_0$} is a subfamily of $E_0$ which has the finite intersection property.\footnote{A family $F \subseteq 2^X$ has the \emph{finite intersection property} when the intersection of every finite subset of $F$ is non-empty.} A body of evidence is maximal if it cannot be properly extended to any other body of evidence.
\end{dfn}

Next we indicate when a body of evidence supports a proposition $P$ \cite{BFS2016,Baltag2016}:

\begin{dfn}[Support of a proposition]\label{def:support}
  Let $\mathcal{M} = (X,E_0,V)$ be an evidence model. A \emph{body of evidence $F \subseteq E_0$} supports a proposition $P$ if and only if $\bigcap F \subseteq P$.
\end{dfn}

The notion of belief \`a la van Benthem and Pacuit is based on the agent's body of evidence. In \cite{BenthemP11}, an agent \emph{believes a proposition $P \subseteq X$} if and only if every \emph{maximal} body of evidence $F \subseteq E_0$ supports $P$ (i.e., $\bigcap F \subseteq P$). According to this definition in evidence logic, the agent is able to combine her available evidence in a maximally consistent way. Yet as explicitly indicated in \cite{Baltag2016}, this does not mean that all her beliefs will be consistent. Indeed, while the mentioned concept of belief works well in the finite case, it can yield inconsistent beliefs in the infinite case:

\begin{example}
  Consider the evidence model $(\mathbb{N}, E_0 = \set{[n,+\infty)\mid n\in \mathbb{N}},\emptyset)$. Note how $E_0$ itself is a body of evidence and, moreover, is the unique maximal one. But $\bigcap E_0 = \emptyset$, and thus the agent believes $\emptyset$.
\end{example}

% To avoid the belief of contradiction, \cite{Baltag2016} propose the topological model for evidence and the topological semantics for belief.\\

Such examples have been the main motivation in \cite{Baltag2016} to provide an improved semantics for evidence-based beliefs which changes the above given definition of van Benthem and Pacuit. Their revised concept of belief is provided in the context of a topological semantics:

\ourparagraph{Topological Semantics for Beliefs} The \emph{topological} approach of \cite{Baltag2016} generalizes the setting of \cite{BenthemP11}, as it introduces a concept of belief which is always consistent. Moreover, the definitions it provides for knowledge and other related epistemic notions have both epistemic and topological significance. We limit ourselves here to the basic doxastic concept:

\begin{dfn}[Topological evidence model \cite{Baltag2016}]\label{def:topo_model}
  A \emph{topological evidence model} $\mathcal{M} = (X, E_0, \tau_{E_0}, V)$ extends an evidence model $(X, E_0, V)$ (Definition \ref{def:evidence_model}) with $\tau_{E_0}$, the topology over $X$ generated by $E_0$.\footnote{A \emph{topology} over $X \neq \emptyset$ is a family $\tau \subseteq 2^X$ containing both $X$ and $\emptyset$, and closed under finite intersections and arbitrary unions. The elements of a topology are called \emph{open sets}. The \emph{topology generated by $E \subseteq X$} is the smallest topology $\tau_E$ over $X$ such that $E\subseteq \tau_E$.} For simplicity, and when no confusion arises, $\tau_{E_0}$ will be denoted simply by $\tau$.
\end{dfn}

A topological evidence model extends an evidence model with a topology: a structure describing the different ways in which the available pieces of evidence can be combined. \shi{Note that pieces of evidence are modeled as opens in the topology; since a topology is closed under \emph{finite} intersection, this fits our intuition about the agent's limited ability for combining evidence. Further reasons for modeling evidence as opens can be found in \cite[Chapter 2]{Vickers1989}, \cite{Kelly1996} and \cite{Baltag2016}.} 
%\footnote{Note how, in some sense, the topology makes explicit the `reasoning abilities' of the agent, and thus shows in which sense the agent is logically omniscient.}
In this setting, Baltag et al. provide an improved version of evidence-based belief:

\begin{dfn}[Evidence-based belief \cite{Baltag2016}] \label{def:topo_model_belief}
  Let $\mathcal{M}$ be a topological evidence model $(X, E_0, \tau, V)$. The agent \emph{believes a proposition $P \subseteq X$} if and only if every non-empty open set can be `strengthened' into an open set supporting $P$ (i.e., for any $t \in \tau \setminus \set{\emptyset}$ there is $t' \in \tau \setminus \set{\emptyset}$ such that $t' \subseteq t$ and $t' \subseteq P$).
\end{dfn}

This notion of belief, which we call \emph{evidence-based belief} ($\Bel$), coincides with the concept of belief \`a la van Benthem and Pacuit in evidence models when the domain is finite. However, $\Bel$ does remain consistent in the infinite case (i.e., in topological evidence models, $\Bel \emptyset = \emptyset$).

{\smallskip}

Based on these ingredients, the next section will extend the topological evidence models by adding tools from formal argumentation theory \cite{dung1995}.

%********************************************************************************************************
\section{Topological Argumentation Models}\label{sec:GAM}
%********************************************************************************************************

Recall how, in the above evidence models of \cite{BenthemP11}, pieces of evidence may be in conflict (i.e., there may be $e_1, e_2 \in E_0$ such that $e_1 \cap e_2 = \emptyset$). This is a natural assumption, as an agent may collect different pieces of evidence (from, e.g., different sources) supporting contradictory facts. The crucial issue is then to find reasonable ways for combining these pieces of evidence or, in other words, to single out meaningful subsets of $E_0$ that can be used to define a notion of belief. In the above mentioned approach, this role is played by maximal bodies of evidence.

Something similar happens in the topological approach of \cite{Baltag2016}: even though \emph{combined} pieces of evidence may be in conflict (there may be $t_1, t_2 \in \tau$ such that $t_1 \cap t_2 = \emptyset$), the authors singled out those whose complement is \emph{nowhere dense} in order to define their topological notion of evidence-based belief.\footnote{See \cite{Baltag2016} for the definition of \emph{dense}, \emph{nowhere dense} and other topological notions, and also for their epistemic interpretation.} This suggests that different doxastic concepts can be provided by singling out alternative (but, of course, reasonable) combined pieces of evidence.

{\medskip}

Let us now focus on specific pieces of evidence that we can think of as examples of an agent's arguments. As readers familiar with the abstract argumentation theory of \cite{dung1995} know, its framework and the various concepts that are defined upon it, can be understood as a collection of tools that allow us to single out acceptable arguments (in our context, amounting to `acceptable' pieces of evidence) from all given arguments. In order to combine such tools with the semantic setting of topological evidence models, one only needs to provide a reasonable definition of what it means for a piece of evidence (or `argument') to attack another, where a piece of evidence is represented as a set of worlds. It is natural to say that two pieces of evidence attack each other when they are in conflict, that is, when their intersection is empty. For example, in a given topological evidence model,
\begin{center}
  for any $t_1, t_2 \in \tau \setminus \set{\emptyset}$: $t_1$ and $t_2$ \emph{attack each other} iff $t_1 \cap t_2 = \emptyset$.
\end{center}
Note that in this example the attack is mutual (i.e., the attack relation is symmetric), yet that does not always need to be the case.

\ourparagraph{From Conflict to Attack} In this paragraph we use an attack relation that is not necessarily symmetric, but it is still related to the mentioned notion of conflict.

\begin{dfn}[Topological argumentation model]\label{def:arg_model}
  A \emph{topological argumentation model} $\mathcal{M} = (X, E_0, \tau_{E_0},$ $\lt, V)$ extends a topological evidence model $(X, E_0, \tau, V)$ (Definition \ref{def:topo_model}) with a relation ${\lt} \subseteq (\tau \times \tau)$, called the \emph{attack relation} on $\tau$ (with $t_1 \lt t_2$ read as \emph{``$t_2$ attacks $t_1$''}), which is required to satisfy the following conditions:
  \begin{enumerate}
    \item for any $t_1, t_2 \in \tau$: $t_1 \cap t_2 = \emptyset$ if and only if $t_1 \lt t_2$ or $t_2 \lt t_1$;

    \item for any $t, t_1, t'_1,  \in \tau$: if $t_1 \lt t$ and $t'_1 \subseteq t_1$, then $t'_1 \lt t$;

    \item for every $t \in \tau \setminus \set{\emptyset}$: $\emptyset \lt t$ and $t \not\lt \emptyset$.
  \end{enumerate}
\end{dfn}

The first condition states that attack implies conflict (right to left), but also that, while conflict implies attack, the implied attack does not need to be mutual (left to right). The second condition asks that, if $t$ attacks $t_1$, then it should also attack any stronger $t'_1$. The last condition establishes that the empty set is attacked by all non-empty opens, and that it does not attack any of them.\footnote{In fact, as the first condition implies, it only attacks itself.}

%***************************************
\subsection{Grounded Semantics for Argument-based Belief}
%***************************************

The attack relation $\lt$ defines, together with its domain $\tau$, an attack graph $\mathcal{A}_\tau = (\tau,\lt)$. Given such a structure, it is possible to use the tools from abstract argumentation theory in order to single out `good' families of combined evidence (or arguments), which in turn will allow us to define a notion of belief grounded in arguments. In order to do that, here is first a quick recollection of the basic concepts of argumentation theory put forward in our framework:

\begin{dfn}[Characteristic (defense) function]
  Let $\mathcal{M}$ be a topological argumentation model, and let $\mathcal{A}_\tau = (\tau, \lt)$ be the attack graph on $\tau$. A subset $T \subseteq \tau$ is said to \emph{defend} $t \in \tau$ if and only if any open $t'$ attacking $t$ (i.e, for all $t' \in \tau$ such that $t \lt t'$) is itself attacked by some open in $T$ (i.e., there is $t'' \in T$ such that $t' \lt t''$). The \emph{characteristic function of $\mathcal{A}_\tau$}, denoted by $d_\tau$ and also called the \emph{defense function}, receives a set of opens $T \subseteq \tau$ and returns the set of opens that are defended by $T$:
  \[ d_\tau(T) := \set{t \in \tau \mid t \text{ is defended by } T} \]
  When $t \in d_\tau(T)$, it is said that $t$ is \emph{acceptable} with respect to $T$.
\end{dfn}

The characteristic function $d_\tau$ is monotonic \cite[Lemma 19]{dung1995}, and therefore it has a least fixed point $\LFP_\tau$ (i.e., $\LFP_\tau$ is the smallest subset of $\tau$ satisfying $\LFP_\tau = d_\tau(\LFP_\tau)$). Note, then, how this least fixed point is a set that can defend all ($\subseteq$) and only ($\supseteq$) its members against any attack. Moreover, it can be proved that the set is also conflict-free (i.e., there are no $t, t' \in \LFP_\tau$ such that $t \lt t'$). Due to these properties, $\LFP_\tau$ provides an excellent candidate for the relevant family of open sets in $\tau$ over which beliefs can be defined. (In abstract argumentation, $\LFP_\tau$ is called \emph{the grounded extension}.) There are, of course, other alternatives on which we can base beliefs: abstract argumentation theory defines also, e.g., the notions of an \emph{admissible set} (a set that is conflict-free and all its arguments are acceptable with respect to it), a \emph{complete set} (one that is admissible and contains every acceptable argument with respect to it), a \emph{preferred set} (a maximal admissible set) and a \emph{stable set} (a conflict-free set that attacks every argument not in it). One of the reasons why the grounded set has been chosen is that, while the other alternatives provide more than one set (there is, in general, more than one admissible/complete/preferred/stable set), the grounded set is always unique\fer{; this guarantees that beliefs will be consistent, a property that is typically desirable.}\footnote{\fer{Of course, consistency can be also achieved when some of the other mentioned sets are used, but this would require the use of an external mechanism, and that might not be desirable for simplicity reasons.}}  One possible drawback of the grounded extension is that, in general, it may be empty; however, this is not the case in our framework, as $X$ is never attacked (it is in conflict only with the empty set, which does not attack anybody) and thus it is always in $\LFP_\tau$. \shi{Moreover, the grounded extension is closely related to the topological notion of belief defined in \cite{Baltag2016}, as we will show later.} Thus we propose the following definition of belief grounded on arguments:

\begin{dfn}[Grounded Belief]\label{def:arg_model_bef}
  Let $\mathcal{M}$ be the topological argumentation model $(X, E_0, \tau_{E_0}, \lt, V)$. The agent \emph{believes a proposition $P \subseteq X$} (notation: $\B P$) if and only if there is an open set in $\LFP_\tau$ supporting $P$, that is
  \[ \B P \quad\text{if and only if}\quad \exists f \in \LFP_\tau: f \subseteq P \]
\end{dfn}

\shi{The definition says that the agent has grounded belief of $P$ if and only if the agent has a ``good" argument for $P$ which is a member of the least fixed point.}

\ourparagraph{Properties of grounded beliefs} From its definition (Definition \ref{def:arg_model_bef}) we observe that grounded beliefs are upward-closed ($\B P$ and $P \subseteq Q$ imply $\B Q$). But a stronger claim also holds: $\LFP_\tau$ itself is closed upwards.

\begin{pro}\label{pro:LFPupwards}
  Given a topological argumentation model, if $f \in \LFP_\tau$ and $f'\in \tau$ is such that $f \subseteq f'$, then $f' \in \LFP_\tau$.
\end{pro}

However, grounded beliefs are not closed under conjunction. To illustrate this fact, consider the topological argumentation model
\begin{equation}\label{equation:1}
\mathcal{M} = ( \set{1,2,3}, E_0 = \set{ \set{1}, \set{2}, \set{3} , \set{1,2}, \set{2,3}}, \tau = 2^X, \lt ,V)
\end{equation}
with $\lt$ an attack relation in which singletons attack each other and $\set{3}\lt \set{1,2}$, $\set{1}\lt \set{2,3}$, $\set{2}\lt \set{1,3}$ and $\set{1,3}\lt \set{2}$, and the empty set is attacked by any sets in $\tau$, as shown in Figure \ref{Counterexample to CUC}.\footnote{Attack edges involving the empty set are not drawn.} Thus, according to the definition, $\LFP_\tau$ is $\set{ \set{1,2}, \set{2,3}, \set{1,2,3} }$, a set that is not closed under intersection.

\begin{figure}[ht]
  \centering
\definecolor{qqqqff}{rgb}{0.,0.,1.}
\definecolor{cqcqcq}{rgb}{0.7529411764705882,0.7529411764705882,0.7529411764705882}
\begin{tikzpicture}[line cap=round,line join=round,>=latex,x=1.0cm,y=1.0cm, scale = 0.6]
\draw [rotate around={0.:(2.,-1.)}] (2.,-1.) ellipse (4.060788100849396cm and 0.7cm);
\draw [rotate around={58.183668383560516:(0.32605438329016084,1.3019196162848687)}] (0.32605438329016084,1.3019196162848687) ellipse (3.359184291413122cm and 1.0965342762477635cm);
\draw [rotate around={-58.18366838356031:(3.67394561670985,1.3019196162848716)}] (3.67394561670985,1.3019196162848716) ellipse (3.359184291413127cm and 1.0965342762477728cm);
\draw [->,line width=0.4pt] (1.9,3.5) -- (-0.85,-0.7);
\draw [->,line width=0.4pt] (-0.85,-0.7) -- (1.9,3.5);
\draw [->,line width=0.4pt] (2.15,3.5) -- (4.8,-0.8);
\draw [->,line width=0.4pt] (4.8,-0.8) -- (2.15,3.5);
\draw [->,line width=0.4pt] (-0.5,-1) -- (4.5,-1);
\draw [->,line width=0.4pt] (4.5,-1.) -- (-0.5,-1.);
\draw [->,line width=0.4pt] (2.,-0.3) -- (2,3.5);
\draw [->] (2.,3.5) -- (2.,-0.3);
\draw [->,line width=0.4pt] (1.4,0.96) -- (4.5,-0.8);
\draw [->] (2.6,0.96) -- (-0.6,-0.8);

\begin{scriptsize}
\draw[color=black] (0.2168514413918739,-2.195804513133274) node {$\{1,3\}$};
\draw[color=black] (-1.4031485586081252,1.8841954868667292) node {$\{1,2\}$};
\draw[color=black] (4.896851441391871,3.06419548686673) node {$\{2,3\}$};
\draw [fill=qqqqff] (2.0241241686646023,3.731716148023746) circle (2.5pt);
\draw [fill=qqqqff] (-1.,-1.) circle (2.5pt);
\draw [fill=qqqqff] (5.,-1.) circle (2.5pt);
\end{scriptsize}
\end{tikzpicture}
  \caption{Grounded beliefs are not closed under conjunction.}
  \label{Counterexample to CUC}
\end{figure}
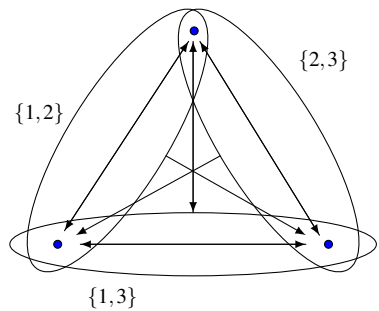

%\shi{Note that in the topological argumentation model of \eqref{equation:1}, the topology is a discrete topology. This type of topology will play a role in our proof of the completeness of our logic. However, it does not mean that the topology in a topological argumentation model has to be the discrete topology.}
%\fer{\itshape $\blacktriangleright$ I would put this paragraph about the discreteness in some other part of the text.}

\shi{To get a better grasp of the whole framework, consider the following scenario which can be modeled by the above topological argumentation model.} 

\begin{example}
  \shi{The zoo in Tom's town bought a new animal and had it under exhibition. Tom was curious about what the animal is, so he asked his colleagues. However, he got different answers from them. Some told him that the animal was a penguin ($\set{1}$), some told him that the animal was a pterosaur ($\set{2}$) and some told him that the animal was a bat ($\set{3}$). Moreover, two other colleagues he really trusts told him that the animal could fly ($\set{2,3}$) and the animal was not a mammal ($\set{1,2}$). After receiving all these pieces of information, Tom was very puzzled. After all, it was very hard to imagine that there could be a pterosaur (although genetic technology has developed fast). So in such a situation, Tom believed that the animal could fly and the animal was not a mammal, but he did not believe that the animal was a pterosaur.}
\end{example}

The lack of closure under conjunction does not indicate that our agent lacks formal reasoning abilities, as is usually the case in weaker epistemic notions. Indeed, as the topology shows, our agent can put her pieces of evidence together in a meaningful way. In section \ref{sec:failclosure}, we further elaborate on this distinction between the agent's reasoning abilities and the closure properties of her grounded beliefs. For now we focus first on the reasons behind the failure of this closure property, by identifying additional conditions under which the property holds.

{\medskip}

\begin{pro}\label{pro:clointer}
  Let $\mathcal{M}$ be a topological argumentation model and $\mathcal{A}_\tau = (\tau,\lt)$ its attack graph.
  \begin{itemize}
  \item If $\lt$ is transitive (i.e. for any $t_1,t_2,t_3\in \tau$, $t_1\lt t_2$ and $t_2\lt t_3$ imply $t_1\lt t_3$), then $\LFP_\tau$ is closed under intersections;
  \item If $\lt$ is unambiguous (i.e. for any $t_1,t_2,t_3\in \tau$, if $t_1 \lt t_2$ and $t_2\lt t_3$, then $t_1\not\lt t_3$ and $t_3\not\lt t_1$), then $\LFP_\tau$ is closed under intersections;
  \item If $\lt$ is symmetric (i.e. for any $t_1,t_2\in \tau$,  if $t_1\lt t_2$ then $t_2\lt t_1$), then $\LFP_\tau$ is closed under intersections.
  \end{itemize}
\end{pro}

Note here how when $\lt$ is symmetric, our notion of belief boils down to the evidence-based belief in topological evidence models of \cite{Baltag2016}, and thus it is closed under conjunction. We will elaborate on this point below. With respect to transitivity, one may wonder whether closure under conjunction still holds under a weaker requirement that asks for the transitivity of the attack relation only when the involved sets are in conflict (i.e., for any $t_1,t_2,t_3\in \tau$, if $t_1\lt t_2$, $t_2\lt t_3$ and $t_1\cap t_3 = \emptyset$, then $t_1\lt t_3$). We defer the answer to this question till Section \ref{sec:failclosure}, where we compare our concept of grounded belief with a probabilistic notion of belief.

{\medskip}

We conclude this section with a comparison between our concept of grounded beliefs and the ev\-i\-dence-based beliefs under the topological semantics of \cite{Baltag2016}. Recall that, in the latter, the agent believes $P$ (denoted by $\Bel P$) if and only if for any $t \in \tau \setminus \set{\emptyset}$ there is $t' \in \tau \setminus \set{\emptyset}$ such that $t' \subseteq t$ and $t' \subseteq P$ (Definition \ref{def:topo_model_belief}). Moreover,

\begin{dfn}[Family of Bodies of evidence, Combined Evidence]
  Following \cite{Baltag2016},
  \begin{itemize}
    \item let $\mathcal{F}$ denote the \emph{family of all bodies of evidence} (see Definition \ref{def:body_evidence});

    \item let $\mathcal{F}^{\mathit{finite}}$ denote the family of all \emph{finite} bodies of evidence;

    \item define $E$ as the family of the combined evidence given by finite bodies of evidence:
    \[ E := \{\bigcap F \subseteq X \mid F \in \mathcal{F}^{\mathit{finite}}\}.\]
  \end{itemize}
\end{dfn}

According to \cite[Proposition 2]{Baltag2016}, evidence-based beliefs can be equivalently defined as $\Bel P$ if and only if there is an open that supports $P$ and is also consistent with any combined evidence given by finite bodies of evidence (i.e., if and only if there is $t \in \tau$ such that $t \subseteq P$ and $t \cap e \neq \emptyset$ for any $e \in E$).

Now, given a topological argumentation model $\mathcal{M}$, let $J_\mathcal{M}$ denote the set of opens that are consistent with any combined evidence given by finite bodies, that is
\[
  J_\mathcal{M} := \{t \in \tau \mid \forall e \in E: e \cap t \neq \emptyset\}
  \qquad
  \footnote{In \cite{Baltag2016}, the elements of \ensuremath{J_\mathcal{M}} are called \emph{justifications}.}
\]
Note that if $\lt$ in the topological argumentation model $\mathcal{M}$ is symmetric, then $\LFP_\tau = J_\mathcal{M}$.

The following result shows that the notion of grounded belief is weaker than the evidence-based belief of \cite{Baltag2016}.

\begin{pro}\label{pro:just}
  Given any topological argumentation model $\mathcal{M}$, $J_\mathcal{M} \subseteq d_\tau(\emptyset)$.
\end{pro}

It follows from this proposition that $J_\mathcal{M}\subseteq \LFP_\tau$, and thus $\Bel P$ implies $\B P$ in any given topological argumentation model. Therefore, grounded beliefs are weaker in the sense of being less restrictive and therefore allow more formulas to be believed by the agent. In the next section we study the logic of beliefs grounded on arguments.

%********************************************************************************************************
\section{The Logic of Beliefs Grounded on Arguments (ABBL)}\label{sec:ABBL}
%********************************************************************************************************

This section studies the logic of the notion of belief grounded on arguments. In order to provide a complete axiomatic characterization of the logic of grounded belief, we work for now with a language that extends the propositional part with a belief operator.

\begin{dfn}[Language]
  The language $\mathcal{L}$ of ABBL is generated by the following grammar, where $p\in \mathsf{Prop}$.
  \[ \varphi ::= p \mid \neg \varphi\mid \varphi\wedge\varphi \mid \B \varphi \]
\end{dfn}

\begin{dfn}[Semantics]
  Given a topological argumentation model $\mathcal{M} = (X, E_0, \tau_{E_0}, \lt, V)$ and a possible world $x \in X$, the truth condition of formulas in $\mathcal{L}$ is defined as usual for atomic propositions and Boolean operators. For the grounded belief modality,
  \begin{center}
    \begin{tabular}{l@{\quad{iff}\quad}l}
      $\mathcal{M},x \models \B \varphi$ & there exists $f \in \LFP_\tau$ such that $f \subseteq \br{\varphi}_\mathcal{M}$
    \end{tabular}
  \end{center}
  with $\br{\varphi}_\mathcal{M} := \set{x \in X\mid \mathcal{M},x \models \varphi}$ the set of $\varphi$-worlds in $\mathcal{M}$.
\end{dfn}

%We will write $\models_{AM}\varphi$ for the validity of $\varphi$ in the class of all argument models.\\

\ourparagraph{Axiom System} Table \ref{tbl:axioms} shows the axiom system ABBS, which below (Theorem \ref{thm:completeness}) will be shown to be sound and complete for the language ABBL with respect to topological argumentation models.

\begin{table}[h!]
  \centering
  \caption{Axiom system ABBS for ABBL}
  \begin{tabular}{r@{\;\;}l@{\qquad}r@{\;\;}l}
    \toprule
    \multicolumn{4}{l}{Propositional Tautologies and Modus Ponens}\\
    4:  & $\B\varphi\rightarrow \B \B\varphi$                                           & 5: & $\neg \B\varphi\rightarrow \B \neg \B\varphi$\\
    RE: & from $\varphi\leftrightarrow \psi$ infer $\B\varphi\leftrightarrow \B\psi$ & D: & $\B\varphi\rightarrow \neg \B\neg \varphi$\\
    M:  & $\B(\varphi\wedge\psi)\rightarrow \B\varphi\wedge \B\psi$                     & N: & $\B\top$\\
    \bottomrule
  \end{tabular}
  \label{tbl:axioms}
\end{table}

The axiom system ABBS extends the well-known axiom system EMN (the propositional fragment, the RE rule and axioms M and N) with axioms 4, 5 and D. It is known that EMN is sound and complete with respect to the class of neighborhood models that are \emph{supplemented} (i.e., closed under conjunction elimination or, equivalently, closed under supersets) and contain the unit (i.e., the domain is in the neighborhood) \cite{chellas:modal:1980}. This suggest a detour for proving the completeness result for our logic with respect to topological argumentation models.

\begin{theorem}\label{thm:completeness}
  For any $\varphi\in \mathcal{L}$ and $\Phi\subseteq \mathcal{L}$,
  \[ \Phi \vdash_{ABBS} \varphi \qquad\text{if and only if}\qquad \Phi \models \varphi \]
  where $\Phi\vdash_{ABBS} \varphi$ indicates that $\varphi$ is derivable from $\Phi$ in the axiom system ABBS and $\Phi \models \varphi$ indicates that, for any topological argumentation model $\mathcal{M}$ and any $x\in X$, if $\mathcal{M},x \models \Phi$ (i.e. $\mathcal{M},x \models \phi$ for all $\phi \in \Phi$) then $\mathcal{M},x\models \varphi$.
\end{theorem}
\begin{proof}[Proof sketch]
  Soundness is straightforward;
%\footnote{In particular, axioms 4 and 5 are valid because the uniformity of the models ($\LFP_\tau$ is the same in every possible world) and the fact that $\LFP_\tau$ is never empty. Axiom $D$ is valid because $\B P$ implies $\exists e \in \LFP_\tau$ such that $e \subseteq P$, and therefore there cannot be $e' \in \LFP_\tau$ with $e' \subseteq X \setminus P$ (since $e \cap e' = \emptyset$, one of them would attack the other). Finally, axiom $M$ is valid because $\LFP_\tau$ is closed upwards (Proposition \ref{pro:LFPupwards}) and axiom $N$ is valid because $X$ is always in $\LFP_\tau$.}
for completeness, the proof uses a modal equivalence result. First, define a \emph{belief neighborhood model} $\mathsf{M}$ as a uniform neighborhood model $(X,N_{\NB},V)$ where the neighborhood function $N_{\NB} \subseteq 2^X$ satisfies the following conditions:
  \begin{itemize}
    \item $X \in N_{\NB}$ ($N_{\NB}$ contains the unit);
    \item if $b \in N_{\NB}$, then $b'\in N_{\NB}$ for any $b'$ such that $b\subseteq b'$ ($N_{\NB}$ is closed under supersets);
    \item if $b \in N_{\NB}$, then $X\setminus b\notin N_{\NB}$ ($N_{\NB}$ does not contain the complement of any of its elements).
  \end{itemize}
  In such structures, the semantic interpretation of a modality for this \emph{neighborhood-based belief} is given by
  \begin{center}
    \begin{tabular}{l@{\quad{iff}\quad}l}
      $\mathsf{M}, x \models \NB \varphi$ & $\br{\varphi}_\mathsf{M} \in N_{\NB}$
    \end{tabular}
  \end{center}

  The following lemma states that every topological argumentation model gives raise to a modally equivalent belief neighborhood model, and vice-versa. Since the axiom system ABBS is sound and complete with respect to belief neighborhood models (see, e.g., \cite{chellas:modal:1980}), it follows that the system is also sound and complete with respect to topological argumentation models.

  \begin{lem}\label{lem:modalequi}
    For any topological argumentation model $\mathcal{M} = (X, E_0,$ $\tau, \lt,V)$ there is a belief neighborhood model $\mathsf{M} = (X,N_{\NB},V)$ with the same domain and atomic valuation such that $\mathcal{M}$ and $\mathsf{M}$ are point-wise modally equivalent with respect to the language $\mathcal{L}$ (i.e. for any $\varphi\in \mathcal{L}$ and any $x\in X$, $\mathcal{M},x\models \varphi$ if and only if $\mathsf{M},x\models \varphi$.) And vice versa.
  \end{lem}

%\footnote{We only need to show that $\lt_{N_{\NB}}$ satisfies the three requirements of Definition \ref{def:arg_model}.\begin{inparaenum}[\itshape (1)] \item Take any $t_1, t_2 \in \tau$. {\leftright} Suppose $t_1 \cap t_2 = \emptyset$. If $t_1 \not\in N_{\NB}$ then, from $\lt_{N_{\NB}}$'s definition, it follows that $t_1 \lt_{N_{\NB}} t_2$. Otherwise $t_1 \in N_{\NB}$; now, from $t_1 \cap t_2 = \emptyset$ it follows that $t_2 \subseteq X \setminus t_1$, and thus $t_2 \in N_{\NB}$ would imply $X \setminus t_1 \in N_{\NB}$ ($N_{\NB}$ is closed under supersets), which would clash with $t_1 \in N_{\NB}$ and the fact that $N_{\NB}$ does not contain the complement of any of its elements. Thus, $t_2 \not\in N_{\NB}$ and therefore, by definition, $t_2 \lt_{N_{\NB}} t_1$. {\rightleft} Suppose either $t_1 \lt_{N_{\NB}} t_2$ or else $t_2 \lt_{N_{\NB}} t_1$; in both cases $\lt_{N_{\NB}}$'s definition gives us immediately $t_1 \cap t_2 = \emptyset$. \item Take $t, t_1, t'_1,  \in \tau$ such that $t_1 \lt_{N_{\NB}} t$ and $t'_1 \subseteq t_1$. The first implies $t_1 \not\in N_{\NB}$, and then $t'_1 \not\in N_{\NB}$ should be the case, as otherwise from the second and $N_{\NB}$'s closure under supersets we would get the contradictory $t_1 \in N_{\NB}$. But the first also implies $t \cap t_1 = \emptyset$, which together with the second yields $t \cap t'_1 = \emptyset$. Hence, from $\lt_{N_{\NB}}$'s definition, $t'_1 \lt_{N_{\NB}} t$. \item The attack conditions on $\emptyset$ are `embedded' in $\lt_{N_{\NB}}$'s definition\end{inparaenum}.}

  For showing modal equivalence, the key is the following lemma:

  \begin{lem}\label{lem:lfpnb}
    $\LFP_\tau = N_{\NB}$
  \end{lem}

   Hence, $\mathsf{M}$ and $\mathcal{M}$ are modally equivalent.
\end{proof}

Thus, our notion of grounded belief can be completely characterized by the axiom system ABBS.

%********************************************************************************************************
\section{Failure of Closure and Rational Belief}\label{sec:failclosure}
%********************************************************************************************************

One of the distinctive features of our notion of grounded belief is its lack of closure under conjunction. This feature is also a distinctive mark of the well-known \emph{probabilistic} notion of belief, where belief is interpreted as high enough probability (typically, above threshold $\frac{1}{2}$). It is thus worthwhile to compare the two approaches.

\begin{dfn}[Probabilistic Belief Model \cite{lenzen:glauben:1980,van:eijck:belief:2014}]
  A probabilistic belief model is a tuple $\mathbf{M} = (X, \mu, V)$ where $X$ is a countable space and $\mu: 2^X\rightarrow [0,1]$ is a probability function satisfying Kolmogorov's Axioms: \begin{inparaenum}[\itshape (1)] \item $\mu(Q) \geq 0$ for any $Q\subseteq X$, \item $\mu(X) = 1$, and \item for any countable sequence of disjoint sets $S_1,S_2,\ldots$, $\mu(\bigcup^\infty_{i=1} S_i) = \Sigma^\infty_{i=1}\mu(S_i)$\end{inparaenum}.
\end{dfn}

\noindent In this probabilistic belief model, $\mu$ is interpreted as agent's credence function which measures the agent's credence in each proposition. Hence credible enough propositions can be taken to represent the agent's beliefs, where ``credible enough" means ``above some threshold". In order to illustrate the relation between our concept of grounded belief and this quantitative counterpart, we consider the threshold to be $0.5$. Thus, the concept of \emph{probabilistic belief} that we consider in this paper is defined as
\begin{center}
\begin{tabular}{l@{\quad\text{iff}\quad}l}
$\mathbf{M},x\models \PB \varphi$ & $\mu(\br{\varphi}_\mathbf{M}) > 0.5$
\end{tabular}
\end{center}

It is important to note that the axiom system ABBS is sound with respect to the probabilistic belief model. Moreover, for any probabilistic belief model we can construct a modally equivalent belief neighborhood model (see the proof of Theorem \ref{thm:completeness}): it is enough to put in the neighbourhood $N_{\NB}$ all subsets of the domain whose probability is greater than $0.5$, i.e. $N_{\NB} = \set{S\subseteq X\mid \mu(S) > 0.5}$. Given Kolmogorov's axioms, the resulting structure indeed satisfies the belief neighborhood requirements. Even more: by Lemma \ref{lem:modalequi}, we can build a topological argumentation model from the belief neighborhood model, which is modally equivalent to the belief neighborhood model and thus modally equivalent to the probabilistic belief model. However, not every topological argumentation model has a modally equivalent probabilistic belief model. This is because the axiom system ABBS is sound but not complete with respect to the probabilistic belief model (c.f. \cite[Section 3.1]{klein:introspection:2015}, \cite{lenzen:glauben:1980} and \cite{van:eijck:belief:2014}).

Thus, probabilistic belief can be seen as a special form of grounded belief, in the sense that the former can be defined through a special attack relation on $2^X$: for any $e,e'\in 2^X$, define $e\lt_\mu e'$ if and only if $e\cap e' = \emptyset$ and $\mu(e) \leqslant \mu(e')$. This attack relation satisfies extra conditions, for example, for any $e_1,e_2,e_3\in 2^X$, if $e_1\lt e_2$, $e_2\lt e_3$ and $e_1\cap e_3 = \emptyset$, then $e_1\lt e_3$. Since the probabilistic notion of belief is not closed under conjunction either, this extra condition does not save the failure of closure under conjunction, which answers the question we posed in Section 2.

%And it is defined upon a $\sigma$-algebra (a family of subsets that contains the unit and it is closed under complements and countable unions), i.e. $\tau = 2^X$,  which is a topology since the space $X$ in a probabilistic belief model $\mathbf{M}$ is required to be countable.

{\medskip}
\ourparagraph{Philosophical Debate.} Let us re-examine one of the most important criticisms faced by probabilistic beliefs from the perspective of our topological argumentation models. The criticism is based on the \emph{conjunction rule} which states that if it is `rational' for an agent to believe $\varphi$ and it is `rational' to believe $\psi$, then it is `rational' to believe $\varphi$ and $\psi$. As stated in \cite{foley:beliefs:2009},
\begin{center}
 \begin{tabular}{c}
   \begin{minipage}[t]{0.9\columnwidth}
     there are [\ldots] prima facie worries associated with rejecting the rule, the most fundamental of which is that if we are not required on pains of irrationality to believe the conjunction of propositions that we rationally believe, we might seem to lose some of our most powerful argumentative and deliberative tools. (p.42)
    \end{minipage}
  \end{tabular}
\end{center}

Thus, according to this (which \cite{foley:beliefs:2009} actually criticizes), a rational agent should follow the conjunction rule. Since probabilistic beliefs violate the rule, philosophers have debated whether probabilistic beliefs are good candidates for representing a rational concept of belief. In this discussion it is essential to note that if ``rational" is viewed as the willingness and capability of following logical rules, it implicitly burdens belief with a requirement about the agent's ability of deductive reasoning.\footnote{Our discussion does not exclude the possibility that `rationality' may well require more than simply reasoning in a logical way. But if extra constraints are imposed the formal setting is bound to change as well.} In \cite{foley:beliefs:2009}, the author argues not only that this worry is misplaced, but also that beliefs should not be burdened with the agents' deductive reasoning abilities. Note that Foley proposes a clear distinction between two reasoning processes, one involving an agent's beliefs (not closed under conjunction) and deductive processes (which are closed under conjunction). What our proposal shows is that the two processes are distinct but can be united in one formal framework.

{\smallskip}

The topological argumentation model shows how the notion of grounded belief is based on the agents' topology of evidence and the attack relation on the topology. So on the one hand, the agent modelled by a topological argumentation model is a powerful logician: she can generate a full topology from her collections of pieces of evidence. On the other hand, this logician' grounded belief does not follow the conjunction rule. So even if the agent has a good argument/reason to believe $\varphi$ and a good argument/reason to believe $\psi$, she may not have a good argument/reason to believe $\varphi \land \psi$. In contrast to many probabilistic belief models where these parts are hidden, our setting makes the different processes explicit.

\section{Conclusion and Future Work}\label{sec:conclusions}
%********************************************************************************************************

This paper studies a notion of grounded belief. It does so by applying ideas from both abstract argumentation theory and the topological semantics for evidence.

From the perspective of abstract argumentation theory, our work can be seen as an attempt of giving structure to the otherwise abstract arguments in an attack graph. Here arguments are instantiated as sets of possible worlds, and thus we can use logical tools to define not only which propositions they support, but also when one of them attacks another. From the perspective of the topological semantics for evidence, our work provides a generalization of the conflict relation between pieces of evidence. The integration of these ingredients from two different fields brings to fruition a notion of \emph{belief grounded on arguments} whose logical behavior is proved to be characterized by the axiom system ABBS. This notion is compared with other forms of belief, including the evidence-based belief of \cite{Baltag2016} and the probabilistic notion of belief in \cite{lenzen:glauben:1980} and \cite{van:eijck:belief:2014}.

One of the most intriguing features of the introduced concept of grounded belief is its failure of closure under conjunction. In the philosophical literature, this failure is widely taken as a failure of rationality, especially when it comes down to the agent's beliefs. However, as our setting shows, even if the conjunction rule is adopted by the agent on the level of \emph{reasoning} about evidence or arguments, it does not need to be adopted to form new \emph{beliefs}.

{\medskip}

In future work we extend our formal language to include operators that address the agent's arguments explicitly. While our current restricted language allows us to characterize the behavior of our notion of belief, it does lack the expressive power to talk about the relationship between beliefs and the arguments supporting it. Another dimension worthwhile to explore further is the understanding of argumentation not as a single agent's inner reasoning mechanism, but rather as a social process that involves several agents. There exists work in argumentation theory dedicated to this dimension, as the study of methods for merging different attack graphs \cite{coste-marquis:merging:2007} shows. This line of work can further be related to the recent developments in social choice theory \cite{EndrissGrandiAIJ2017}. In our context, the question of merging different agent's attack graphs can shed new light at group-related notions, in particular, distributed belief, which can then be understood from the perspective of abstract argumentation theory.

Finally, while we have restricted our attention to the concept of ground\-ed belief, epistemologists will be interested to see a connection to the seminal studies of knowledge, belief, justifications and their relation (e.g., \cite{lehrerTheory1990}, where knowledge is understood as belief with a correct justification that cannot be defeated by any further true justification). There exist new proposals in this direction (e.g., \cite{Baltag2016}), which, with the help of argumentation theory can shed more light on different epistemic theories and, in particular, on the different notions of knowledge.

\bigskip

\noindent \textbf{Acknowledgments} We thank Johan van Benthem and Aby\"{u}ke \"{O}zg\"{u}n for feedback and discussions on the topic of this paper. We also thank the anonymous referees of TARK 2017 for their useful comments.

\appendix

%********************************************************************************************************
\section*{Appendix}
%********************************************************************************************************

\begin{ap}{Proof of Proposition \ref{pro:LFPupwards}}
  Take any $f \in \LFP_\tau$ and any $f' \in \tau$ such that $f \subseteq f'$. Suppose no one attacks $f'$; then we are done as, by $\LFP_\tau$'s definition, every non-attacked element of $\tau$ should be in $\LFP_\tau$. Suppose otherwise, and let $t$ be one of such opens attacking $f'$; it is enough to find a $f'' \in \LFP_\tau$ attacking $t$, as then $f'$ would be defended by someone in $\LFP_\tau$ and thus, by definition, $f'$ would be in $\LFP_\tau$. Now, since $t$ attacks $f'$, it should also attack the stronger $f$ (as required by $\lt$'s definition); but then, since $f$ is in $\LFP_\tau$, it should be defended by someone in $\LFP_\tau$, that is, there is a $f''$ in $\LFP_\tau$ attacking $t$. This completes the proof.
\end{ap}

\begin{ap}{Proof of Proposition \ref{pro:clointer}}
  The following lemma will be useful.

  \begin{lem}\label{lem:conjunctionequi}
    Let $\mathcal{M} = (X, E_0, \tau_{E_0}, \lt, V)$ be a topological argumentation model. Then, for any $f_1, f_2 \in \LFP_\tau$ , $f_1, f_2 \in \LFP_\tau$ implies $f_1 \cap f_2 \in \LFP_\tau$ if and only if for any $t\in \tau$, if $f_1\cap f_2 \lt t$, then $t\cap f = \emptyset$ for some $f\in \LFP_\tau$.
  \end{lem}
  \begin{proof}
  {\leftright} From left to right, take arbitrary $f_1,f_2\in \LFP_\tau$. Suppose there is an open $t \in \tau$ such that $t$ attacks $f_1\cap f_2$ but is not in conflict with anybody in $\LFP_\tau$. From the latter it follows that nobody in $\LFP_\tau$ attacks $t$, and thus the attacked $f_1 \cap f_2$ is not defended by $\LFP_\tau$; therefore, $f_1 \cap f_2$ is not in $\LFP_\tau$. {\rightleft} From right to left, take arbitrary $f_1,f_2\in \LFP_\tau$. If there is no $t\in \tau$ such that $f_1\cap f_2\lt t$, then we are done. Otherwise, take arbitrary $t\in \tau$ such that  $f_1\cap f_2\lt t$. It implies that there is $f'\in \LFP_\tau$ such that $t\cap f' = \emptyset$. Thus either $f'$ attacks $t$ or else $t$ attacks $f'$. The former case implies that there is $f''\in \LFP_\tau$ such that $t\lt f''$ by virtue of $f'$'s membership in $\LFP_\tau$; together with the latter case, i.e. $t\lt f'$, we can conclude that there is $f\in \LFP_\tau$ such that $t\lt f$. Hence for any $t\in \tau$ such that $f_1\cap f_2\lt t$, there is $f\in \LFP_\tau$ such that $t\lt f$, which implies that $f_1\cap f_2\in \LFP_\tau$.
  \end{proof}

  Now, for proposition 3.2. The proof for transitivity and the relation being unambiguous, proceeds by contraposition, so take any $f_1, f_2 \in \LFP_\tau$ such that $f_1 \cap f_2$ is not in $\LFP_\tau$. Then, by Lemma \ref{lem:conjunctionequi}, there is an open $t \in \tau$ who attacks $f_1 \cap f_2$ (i.e., $f_1 \cap f_2 \lt t$) and who is not in conflict with elements of $\LFP_\tau$ (i.e., $f \in \LFP_\tau$ implies $t \cap f \neq \emptyset$). The goal is to show that $\lt$ is not transitive.

  {\smallskip}

  Define $t_1, t_2$ and $t_3$ as
  \[ t_1 := f_1 \cap t, \qquad t_2 := f_2 \cap t, \qquad t_3 := f_1 \cap f_2, \]
  and note that none of them are empty. Note also that due to the fact that $t$ attacks $f_1 \cap f_2$ ($f_1 \cap f_2 \lt t$), $t$ must be in conflict with $f_\cap f_2$ ($(f_1\cap f_2)\cap t = \emptyset$); hence, $t_1 \cap t_2 = t_2 \cap t_3 = t_3 \cap t_1 = \emptyset$.

  {\smallskip}

  For transitivity, consider now the following two mutually exclusive and collectively exhaustive cases:\begin{inparaenum}[\itshape (1)] \item $t_1 \lt t_3$ or $t_2 \lt t_3$, \item $t_1 \not \lt t_3$ and $t_2 \not \lt t_3$\end{inparaenum}. In the first, assume the leftmost disjunct $t_1 \lt t_3$; that, together with $t_3 \lt t$ (recall: $t_3 = f_1 \cap f_2$) and $t_1 \not \lt t$ (as $t_1 \cap t = f_1 \cap t \neq \emptyset$) shows that $\lt$ is not transitive. If the rightmost disjunct $t_2 \lt t_3$ is assumed, a similar reasoning yields $t_3 \lt t$ and $t_2 \not \lt t$, and thus lack of transitivity again.

  Now the second case, which implies $t_3 \lt t_1$ and $t_3 \lt t_2$, as the $t_i$ are all in conflict. Since $t_1$ and $f_2$ are in conflict, at least one of them should attack the other. If $t_1 \lt f_2$, then this together with $t_3 \lt t_1$ and $t_3 \not\lt f_2$ (as $t_3 \cap t = f_1 \cap f_2 \neq \emptyset$, and thus they are not in conflict shows that $\lt$ is not transitive. Otherwise, $f_2 \lt t_1$ should be the case; hence, since $f_2$ is in $\LFP_\tau$, it should be defended from $t_1$ by some $f \in \LFP_\tau$, that is, $t_1 \lt f$. But then we have $f_2 \lt t_1$, $t_1 \lt f$ and $f_2 \not\lt f$ (since $\LFP_\tau$) is conflict-free. Thus, again, $\lt$ is not transitive.

  {\smallskip}

  For the relation being unambiguous, consider two cases, $t_1\lt t_2$ or $t_2\lt t_1$. In the case of $t_1\lt t_2$, if $t_2\lt t_3$, then no matter $t_1\lt t_3$ or $t_3\lt t_1$, it is not unambiguous. If $t_3\lt t_2$, then no matter $t_1\lt t_3$ or $t_3\lt t_1$, $\lt$ is not unambiguous. The proof for the case of $t_2\lt t_1$ follows a similar argument. Therefore, we reach the conclusion that $\lt$ cannot be unambiguous.

  {\smallskip}

  For symmetry, assume that $\lt$ is symmetric. Observe that $\LFP_\tau = \{t \in \tau \mid \forall e \in E: e \cap t \neq \emptyset\}$, which is closed under conjunction.
\end{ap}

\begin{ap}{Proof of Proposition \ref{pro:just}}
  Take any $t$ in $J_\mathcal{M}$. Since $t$ is consistent with any $e \in E$, it is also consistent with any open $t' \in \tau$, that is, for any such $t'$ we have $t \cap t' \neq \emptyset$. But then $t$ is not attacked at all, and thus it is defended by the empty set, i.e. $t \in d_\tau(\emptyset)$. Therefore, $J_\mathcal{M} \subseteq d_\tau(\emptyset)$.
\end{ap}

\begin{ap}{Proof of Lemma \ref{lem:modalequi}}
  Given a topological argumentation model, it is easy to build a point-wise modally equivalent belief neighborhood model: it is enough to define the neighborhood function $N_{\NB}$ as $\LFP_\tau$ plus all its element's supersets ($N_{\NB} := \set{ b \in 2^X \mid f \subseteq b \text{ for some } f \in \LFP_\tau }$). For the other direction, given a belief neighborhood model $\mathsf{M} = (X,N_{\NB},V)$, let $\mathcal{M}_\mathsf{M} = (X, E_0, \tau, \lt_{N_{\NB}}, V)$ be the topological argumentation model that shares domain and atomic valuation with $\mathsf{M}$, and in which the family of pieces of evidence is given by the singletons in $X$ ($E_0 := \set{ \set{x} \in 2^X \mid x \in X }$) and therefore the generated topology is the power set of the domain ($\tau = 2^X$). Moreover, define the attack relation $\lt_{N_{\NB}}$ as
  \[
    t \lt_{N_{\NB}} t'
      \quad\text{iff}\quad
      \left\{
        \begin{array}{ll}
          t \cap t' = \emptyset \text{ and } t \notin N_{\NB} & \text{when } t' \neq \emptyset \\
          t = \emptyset                                    & \text{when } t' = \emptyset\\

        \end{array}
      \right.
  \]
  for every $t, t' \in \tau$, so a non-empty $t'$ attacks a non-empty $t$ if and only if they are in conflict and $t$ is not in $N_{\NB}$, and while the empty set does not attack non-empty sets, it is attacked by everybody (including itself). It is easy to verify that this model is a topological argumentation model.
\end{ap}

\begin{ap}{Proof of Lemma \ref{lem:lfpnb}}
  {\prsupset} Take any $b \in N_{\NB}$. By $\lt_{N_{\NB}}$'s definition, no element of $\tau$ who is in conflict with $b$ attacks it (i.e., there is no $t \in \tau$ such that $b \cap t = \emptyset$ and $b \lt_{N_{\NB}} t$); therefore, no one attacks $b$ and hence $b \in \LFP_\tau$.

  {\prsubset} Take an arbitrary $t \in \tau$ such that $t \notin N_{\NB}$. If $t$ is not in conflict with any open in $\tau$, then $t$ must be $X$, so $X \not\in N_{\NB}$, contradicting the fact that $N_{\NB}$ contains the unit. Hence, $t$ should have a conflict with some open, that is, there is at least one $t' \in \tau$ such that $t \cap t'= \emptyset$. From this and $\lt_{N_{\NB}}$'s definition it follows that $t'$ attacks $t$, $t \lt_{N_{\NB}} t'$.

  If at least one of such $t'$ is in $N_{\NB}$, then it is also in $\LFP_\tau$ (the {\prsupset} case above), and thus the fact that $t'$ attacks $t$ implies that $t \not\in \LFP_\tau$, as this set is conflict-free. Then we are done. Otherwise no such $t'$ is in $N_{\NB}$, that is, every open in conflict with $t$ is outside $N_{\NB}$; in other words, no element of $N_{\NB}$ is in conflict with $t$ ($b \in N_{\NB}$ implies $t \cap b \neq \emptyset$).

  Define now the set $s$ as the union of the non-$t$ parts of every element of $N_{\NB}$, that is, $s := \bigcup_{b \in N_{\NB}} (b \setminus t)$. Since $\tau$ is the full power set, $s$ is in $\tau$. This $s$ is in conflict with $t$ ($s \cap t = \emptyset$, by construction), but it is consistent with all $b \in N_{\NB}$ ($b \cap s \neq \emptyset$ for all such $b$, as otherwise the $b' \in N_{\NB}$ in conflict with $s$ would have been a subset of $t$, and thus closure under supersets of $N_{\NB}$ would imply $t \in N_{\NB}$, contradicting our initial assumption). These two facts tells us that we have found an open $s \in \tau$ that attacks $t$ ($t \lt_{N_{\NB}} s$, from the first fact, $t \notin N_{\NB}$ and $\lt_{N_{\NB}}$'s definition) but it is not attacked by anyone in $N_{\NB}$ ($b \in N_{\NB}$ implies $s \not \lt_{N_{\NB}} b$, from the second fact and $\lt_{N_{\NB}}$'s definition). Moreover, it attacks every open not in $N_{\NB}$ with whom it has conflict ($x \notin N_{\NB}$ implies $x \cap s = \emptyset$, again from $\lt_{N_{\NB}}$'s definition). Thus, $s$\begin{inparaenum}[\itshape (1)] \item attacks everybody in $\tau$ that is not in $N_{\NB}$, and \item is not attacked by anybody in $N_{\NB}$\end{inparaenum}; hence no one outside $N_{\NB}$ can be protected by $N_{\NB}$ against $s$, which means that $d_\tau(N_{\NB}) = N_{\NB}$. Together with the fact that $N_{\NB}\subseteq \LFP_\tau$, it implies that $N_{\NB} = \LFP_\tau$. In particular, since $t \not\in N_{\NB}$, we get the required $t \not\in \LFP_\tau$.
\end{ap}

\bibliographystyle{eptcs}
\bibliography{argumentlogic}

\end{document}